\documentclass{article} % For LaTeX2e
\usepackage{iclr2025_conference,times}

\usepackage{amsmath,amsfonts,bm}

\def\eqref#1{equation~\ref{#1}}
\def\1{\bm{1}}

\DeclareMathAlphabet{\mathsfit}{\encodingdefault}{\sfdefault}{m}{sl}
\SetMathAlphabet{\mathsfit}{bold}{\encodingdefault}{\sfdefault}{bx}{n}

\usepackage{hyperref}
\usepackage{url}
\usepackage{graphicx} 
\usepackage{algorithmic,algorithm2e}
\usepackage{amsthm} % Load the package for theorem environments

\newtheorem{theorem}{Theorem}[section]

\newtheorem{corollary}[theorem]{Corollary}
\newtheorem{proposition}[theorem]{Proposition}
\theoremstyle{remark}
\newtheorem{remark}[theorem]{Remark}
\theoremstyle{definition}
\newtheorem{definition}[theorem]{Definition}

\title{The Space Between: On Folding, Symmetries and Sampling}

\author{Michal Lewandowski$^\dagger$, Bernhard Heinzl$^\dagger$, Raphael Pisoni$^\dagger$, Bernhard A.Moser$^{\dagger*}$ \\
$^\dagger$Software Competence Center Hagenberg (SCCH)\\
$^*$Johannes Kepler University of Linz (JKU)\\
\texttt{\{name.surname\}@scch.at} \\
}

\iclrfinalcopy % Uncomment for camera-ready version, but NOT for submission.
\begin{document}

\maketitle

\begin{abstract}
Recent findings suggest that consecutive layers of neural networks with the ReLU activation function \emph{fold} the input space during the learning process. While many works hint at this phenomenon, an approach to quantify the folding was only recently proposed by means of a space folding measure based on Hamming distance in the ReLU activation space. We generalize this measure to a wider class of activation functions through introduction of equivalence classes of input data, analyse its mathematical and computational properties and come up with an efficient sampling strategy for its implementation. Moreover, it has been observed that space folding values increase with network depth when the generalization error is low, but decrease when the error increases. This underpins that learned symmetries in the data manifold  (e.g., invariance under reflection) become visible in terms of space folds, contributing to the network's generalization capacity. Inspired by these findings, we outline a novel regularization scheme that encourages the network to seek solutions characterized by higher folding values. 
\end{abstract}

\section{Introduction}
\label{sec:intro}
Recent works in machine learning indicate that neural networks \emph{fold} the input space during training~\citep{montufar2014number,keup2022origami}. This phenomenon draws inspiration from the way certain natural structures—such as proteins and amino acids—fold to encode information efficiently~\citep{dill2008protein,jumper2021alphafold}. Building on these ideas, \citet{lewandowski25spacefolds} proposed a range-based measure in the discrete activation space of ReLU networks to quantify how \emph{much} a network folds its input space as it learns. Their analysis focuses on deviations from convexity when mapping a straight-line path in the input space to the Hamming activation space: While Euclidean distances increase monotonically along the path, the corresponding Hamming distances may decrease, signaling a folding effect (cf. Fig.~\ref{fig:straight_path_walks} right and Fig.~\ref{fig:convex_dev}).
Originally developed for ReLU networks, this approach leverages the fact that the ReLU activation function partitions the input space into disjoint linear regions~\citep{makhoul89firstlinearregions, montufar2014number}.

In our paper, we firstly show that these regions correspond to equivalence classes defined by the pre-images of either $\{0\}$ or the strictly positive interval $(0,\infty)$. Extending $\{0\}$ to $(-\infty, 0]$ provides a straightforward generalization to a broader class of activation functions that accommodate negative values, including Swish~\citep{ramachandran2018searching}, GELU~\citep{hendrycks2016gaussian}, and SwiGLU~\citep{shazeer2020glu}.
Secondly, we focus on characterizing properties of the space folding measure 
$\chi$, which do also hold in the general case. Thirdly, since computing  $\chi$ relies on sampling from different activation regions, we introduce a non-parametric sampling algorithm that exploits the structure of the aforementioned equivalence classes, thereby reducing redundant computations. 
Lastly, we leverage the fact that space folding values have been observed to \textit{increase} with network depth when the generalization error is low, but \textit{decrease} when the error increases~\citep{lewandowski25spacefolds}. We thus hypothesize the increased folding  contribute to the network's generalization capacity, and hint at a novel regularization strategy that  applies the folding measure at regular intervals (e.g., every $n$  training epochs) to induce stronger folding in the early stages and diminish its influence later in training. 
Our contributions are as follows.

\begin{itemize}
  \item We generalize  the space folding measure beyond the ReLU activation function. Our approach relies on the fact that the pre-image of the partition $\{ (-\infty, 0], (0, \infty)\}$ divides the domain into two connected sets, $f^{(-1)}((-\infty, 0])$ and $f^{(-1)}((0, \infty))$.
  % The characterizing property is that the pre-image $f^{(-1)}$ of the range partition $\{ (-\infty, 0], (0, \infty)\}$ gives a partition of $f^{(-1)}(\mathbb{R})$ consisting of two connected sets, i.e., $f^{(-1)}((-\infty, 0])$ and $f^{(-1)}((0, \infty))$ are connected, where $f^{(-1)}(C) = \{x:\, f(x) \in C\}$ denotes the pre-image of the range set $C$ under $f$. 
  \item We state and prove general properties of the folding measure,  such as (\textit{i}) its stability under traversing different activation regions (Proposition~\ref{prop:traversing_same_region}), (\textit{ii}) the sufficient and necessary, i.e., characterizing, condition for flatness (Proposition~\ref{prop:flatness}), (\textit{iii}) its sensitivity to the direction of the path (Remark~\ref{remark:asymmetricity}), (\textit{iv}) invariance of flatness to direction of path (Corollary~\ref{lemma:looped_path}).
  \item We propose a parameter-free sampling strategy in the Hamming activation space that limits steps within the same equivalence class to reduce redundancy, and we analyze its computational complexity. 
  \item We introduce a new regularization procedure for training neural networks by penalizing low space folding values.
\end{itemize}

The remainder of the paper is organized as follows. Sec.~\ref{sec:related_work} details related work; Sec.~\ref{sec:preliminaries} introduces necessary concepts and fixes notation for the rest of the paper; Sec.~\ref{sec:space_folding} recalls the definition of the space folding measure and then provides its detailed analysis paired with the introduction of the global folding measure; Sec.~\ref{sec:SamplingAlgorithm} introduces a  sampling technique from activation paths along a 1D path which relies on the Hamming distances between samples; Sec.~\ref{sec:folding_regularization} proposes a novel regularization scheme; Sec.~\ref{sec:discussion} summarizes our paper and outlines future research directions.

\section{Related Work}
\label{sec:related_work}

\paragraph{Folding.} The idea of  folding the (input) space has been investigated, among others, in computational geometry~\citep{demaine1998folding_and_cutting}. In context of neural networks, \citet{montufar2014number} in Section 2.4 argued that each hidden layer in a ReLU neural network acts as a folding operator, recursively collapsing input-space regions. 
In~\citet{phuong2020functional}, in the Appendix A.2 the authors defined the folds  by ReLU  networks, but left the exploration quite early on. \citet{lewandowski25spacefolds} proposed the first measure to  quantify the folding by ReLU neural networks. Our approach builds on the proposition therein, and is further motivated by the observation that folding gives rise to symmetries as discussed below. 

\paragraph{Symmetries.}
The modern study of symmetries (in physics) was initiated by~\citet{Noether1916}, who linked  them to \textit{conservation} laws: energy to time translation, momentum to space translation, and angular momentum to rotational symmetry. In the context of machine learning, researchers working with object recognition emphasised the importance
of learning representations that are \textit{invariant} to transformations, e.g.,~\citep{Krizhevsky2012ImageNetCompetition}. Somewhat implicitly, symmetries have been at the core of some of the most
successful deep neural network architectures, e.g., CNNs~\citep{fukushima1980neocognitron,lecun1989backpropagation} are equivariant
to translation invariance characteristic of image classification tasks, while GNNs~\citep{battaglia2018relational}  are equivariant to the full group of permutations
(see~\citet{Higgins2022SymmetryBasedRF} for a detailed overview). 
%\citeauthor{Hashimoto2024symmetriesunification} interpreted parametric redundancies of  ML models as gauge symmetries.  
Our work analyzes symmetries (reflection groups) that arise by space folding and their impact on the generalization capacity of the model.

\paragraph{Linear Regions Sampling.} 
Analyzing neural network linear regions is challenging. Early work bounded their number as a measure of expressivity in ReLU MLPs~\citep{montufar2014number,raghu2017expressive,serra2018bounding,montufar2021sharp}, later extending to CNNs~\citep{xiong2020numberLR_CNN} and GNNs~\citep{chen2022lower}. Empirical studies indicate that linear regions are denser near training data~\citep{zhang2020empirical}, yet standard sampling methods (e.g., Monte Carlo or Sobol sequences~\citep{sobol1967distribution}) often miss small regions. \citet{goujon2024number} showed that along one-dimensional paths, nonlinearity points scale linearly with depth, width, and activation complexity, while \citet{gamba22arealllinearregions} proposed a direction-based method that requires costly minimal step calculations. In contrast, we introduce a sampling strategy in the Hamming activation space to efficiently identify linear regions along $\mathbf{d}=\mathbf{x}_2-\mathbf{x}_1$.

\section{Preliminaries}\label{sec:preliminaries}
We define a \emph{ReLU neural network} $\mathcal{N}:\mathcal{X}\rightarrow \mathcal{Y}$ with the total number of $N$ neurons as an alternating composition of   the ReLU function   $\sigma(x) := \max(x, 0)$   applied element-wise on the input $x$, and  affine functions with weights $W_k$ and biases $b_k$ at layer $k$. 
An input $x\in\mathcal{X}$ propagated through $\mathcal{N}$  generates non-negative activation values on each neuron.
% \textbf{Binarization}: 
A \textit{binarization} is  a mapping $\pi:\mathbb{R}^N \to \{0,1\}^N$ applied to a vector  $v=(v_1,\ldots,v_N)\in\mathbb{R}^N$, resulting in a binary vector by clipping strictly positive entries of $v$  to 1, and non-positive entries to 0, that is $\pi(v_i)=1$ if $v_i>0$, and $\pi(v_i)=0$ otherwise. In our case, the vector $v$ is the concatenation of all neurons of all hidden layers, called an \emph{activation pattern},  and it represents an element in a binary hypercube $\mathcal{H}^N:=\{0,1\}^N$ where the dimensionality is equal to the number $N$ of (hidden) neurons in network $\mathcal{N}$. A 
\emph{linear region} is an element of a partition covering the input domain where the network behaves as an affine function (Fig.~\ref{fig:straight_path_walks}, left).  The Hamming distance, $d_H(u,v):=\left|\{u_{i}\neq v_{i}\text{ for } i=1,\ldots,N\}\right|$, measures the difference between $u,v\in\mathcal{H}^N$, and for binary vectors is equivalent to the $L_1$ norm between those vectors. Lastly, as we will deal with paths of activation patterns, we denote the operation of joining those paths with the operator $\oplus:\mathcal{H}^{k\cdot N}\times\mathcal{H}^{(n-k+1)\cdot N}\to\mathcal{H}^{n\cdot N}$ such that $\{\pi_1,\ldots,\pi_{k}\} \oplus 
 \{\pi_k,\ldots,\pi_n\}=\{\pi_1,\ldots,\pi_k,\ldots,\pi_{n}\}$.
The operation $\oplus$ is defined for connected paths, where the last activation pattern of one path matches the first activation pattern of the other.

\begin{figure}
    \centering
    % \includesvg[width=0.7\textwidth]{figs/input_vs_activation_space.svg}
    \includegraphics[width=0.7\linewidth]{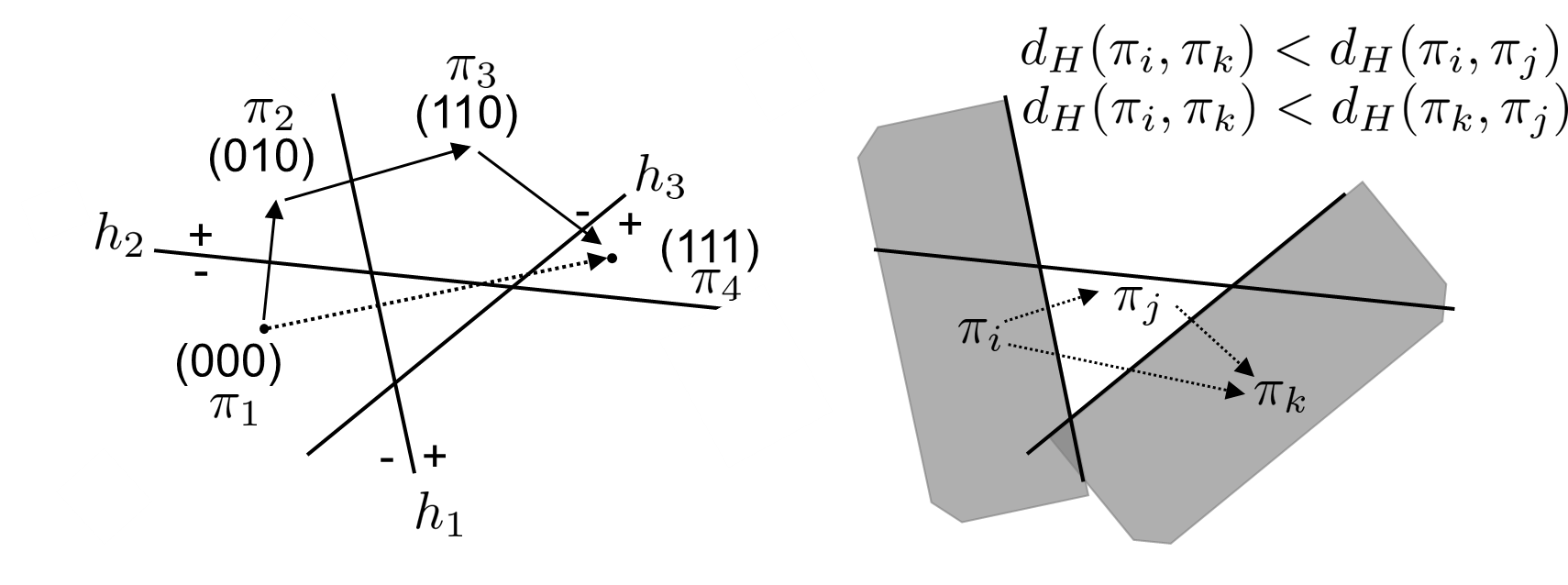}
    \caption{\textbf{Left:} Illustration of a walk on a straight path in the Euclidean input space and the Hamming activation space. The dotted line represent the shortest path  in the Euclidean space. The arrows represent \emph{a} shortest path in the Hamming distance between activation patterns $\pi_1$ and $\pi_4$ (in the Hamming space the  shortest path is not unique). \textbf{Right:} Symmetry in the activation space: gray regions are closer to each other in the Hamming distance than to the region $\pi_j$ that lies between them.}
    \label{fig:straight_path_walks}
\end{figure}

\section{Space Folding}
\label{sec:space_folding}

\subsection{Construction} Consider a straight line connecting two input points $\mathbf{x}_1,\mathbf{x}_2$ in the Euclidean input space. The intermediate points are realized by varying the parameter $t$ in a convex combination $(1- t) \mathbf{x}_1  +  t\mathbf{x}_2$. %\footnote{The approach of~\citet{Fawzi2018empiricalstudyoftopology} is remotely similar to ours as they  used 1D walk between input data to explore whether there exists a continuous path that remains in the decision region betweenany two points of the same label. However,  their analysis does not extend to the activation space which is the basis of our approach.} 
Due to practicality,~\citet{lewandowski25spacefolds}  spaced the parameter $t$ equidistantly on $[0,1]$, creating $n$ segments. Equal spacing, though easy and fast to implement, frequently results in suboptimal choice of the intermediate points (we address this issue in Sec.~\ref{sec:SamplingAlgorithm}). To obtain a walk through activation patterns, we map the straight line $[\mathbf{x}_1,\mathbf{x}_2]$ through a neural network $\mathcal{N}$ to a  \emph{path} $\Gamma:=\{\pi_1,\ldots,\pi_n\}\in\mathcal{H}^{n\cdot N}$ in the Hamming activation space, where the intermediate activation patterns belong to a binary hypercube, $\pi_i\in\mathcal{H}^N$ for all $i\in\{1,\ldots,n\}$ (see Fig.~\ref{fig:convex_dev}).
We consider a change in the Hamming distance with respect to the initial activation pattern  $\pi_1$  at each step $i$, $\Delta_i := d_H(\pi_{i+1}, \pi_1) - d_H(\pi_{i}, \pi_1)
$, and then look at the maximum of the cumulative change $\max_k \sum_{i=1}^k |\Delta_i|$ along the path $\Gamma$, 
\begin{equation}
\label{eq:cum_max}
    r_1(\Gamma) = \max_{i}\sum_{j=1}^{i}\Delta_j= \max_{i} d_H(\pi_i,\pi_1).
\end{equation}
We further keep track of  the total distance traveled on the hypercube when following the path, 
\begin{equation}
\label{eq:eff_path}
 r_2(\Gamma) = \sum_{i=1}^{n-1}d_H(\pi_i, \pi_{i+1}).
\end{equation}
For a measure of \emph{space flatness}, we consider the ratio $r_1(\Gamma)/r_2(\Gamma)$. Equivalently,  the \emph{space folding} measure equals
\begin{equation}
\label{eq:measureFold}
\chi(\Gamma):= 1-\max_{i} d_H(\pi_i,\pi_1) \big/
\sum_{i=1}^{n-1} d_H(\pi_{i}, \pi_{i+1}).
\end{equation}
The folding measure is lower and upper   bounded, $\chi\in[0,1]$; it equals 0 if there are no folds in the activation space, and it converges towards 1 for a path $\Gamma=\{\pi_1,\pi_2,\pi_1,\pi_2,\ldots\}$ looped between two activation regions such that $r_1(\Gamma)=d_H(\pi_1,\pi_2)=c\in\mathbb{R}^+$ and $r_2(\Gamma)\to\infty$. Although theoretically possible, this edge case example might be not
realizable in practice.

\begin{figure}
    \centering
    \includegraphics[scale=0.45]{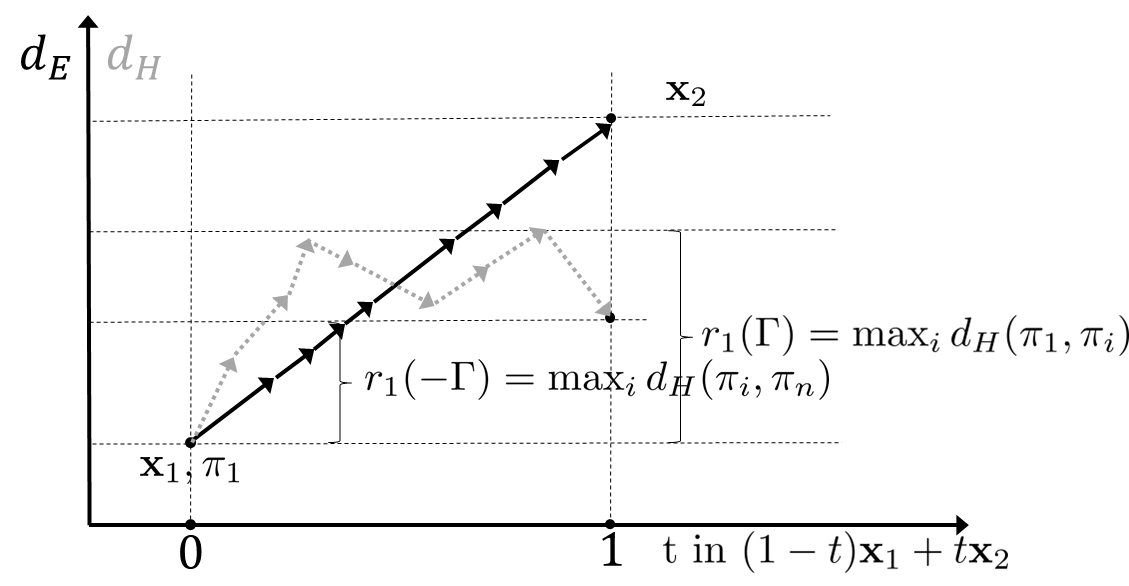}
    \caption{1D straight walk  from $\mathbf{x}_1$ to $\mathbf{x}_2$ in the Euclidean space (black full arrows) and the Hamming activation space (gray dotted arrows). Observe that in the Hamming activation space it might happen that $d_H(\pi_1,\pi_n)<\max_id_H(\pi_1,\pi_i)$, which indicates space folding. The steps are optimized to visit each equivalence class exactly once (not equidistant).}
    \label{fig:convex_dev}
\end{figure}

\subsection{Properties}

In this section, we prove several properties of the folding measure, starting with emphasizing the importance of an appropriate sampling strategy.
We show that taking multiple steps in the same activation region multiple times does not change the measure $\chi$. 
\begin{proposition}[\textbf{Stability}]
\label{prop:traversing_same_region}
    Multiple steps in the same activation region do not influence the space folding measure $\chi$.
\end{proposition}
\begin{proof}
     Consider  a path  $\{\pi_1,\ldots,\pi_i,\ldots,\pi_{i+l}\ldots, \pi_n\}$, where $\pi_i=\pi_{i+1}=\ldots=\pi_{i+l}$ and all other activation patterns are distinct. Observe  that 
$$\max_{j\in\{1,\ldots,i\}}d_H(\pi_1,\pi_j)=\max_{j\in\{1,\ldots,i+l\}}d_H(\pi_1,\pi_j)$$
and $$\sum_{j=1}^{i-1} d_H(\pi_j,\pi_{j+1})=\sum_{j=1}^{i+l-1}d_H(\pi_j,\pi_{j+1}),$$ 
thus traversing the same activation pattern more that once does not change the folding measure. %Indeed,  removing the repeated activation patterns  does not affect the folding measure, what finishes the proof.
\end{proof}

Proposition~\ref{prop:traversing_same_region} further highlights the importance of the sampling strategy that visits every  activation pattern between samples. 
In Sec.~\ref{sec:SamplingAlgorithm}, we propose  a sampling algorithm that relies on the Hamming distances between samples  in the Hamming activation space.
We now show a necessary condition for space flatness, relevant for the upcoming analysis of the direction of folding.
\begin{proposition}[\textbf{Flatness}]
\label{prop:flatness}
$\chi(\Gamma)=0$ if and only if $d_H(\pi_1,\pi_i)$ is non-decreasing for $i=1,\ldots, n$ along the path $\Gamma$.
\end{proposition}
\begin{proof}
We show that  the space flatness implies that $d_H(\pi_1,\pi_i)$ is non-decreasing along the path $\Gamma$. Note that $\chi(\Gamma)=0\Rightarrow \max_jd_H(\pi_1,\pi_j)=\sum_{i=1}^{j-1}d_H(\pi_i,\pi_{i+1})$ for every index $j\in\{1,\ldots,n-1\}$. Let us argue through contradiction: Suppose that $d_H(\pi_1,\pi_i)$ decreases along the path $\Gamma$ for some index $i^*$, i.e., $d_H(\pi_1,\pi_{i^*})<d_H(\pi_1,\pi_{i^*-1})$. This contradicts flatness, since $\sum_{i=1}^{i^*}d_H(\pi_1,\pi_i)>\sum_{i=1}^{i^*-1}d_H(\pi_1,\pi_i)$ while $\max_{j\le i^*-1}d_H(\pi_1,\pi_j)=\max_{j\le i^*}d_H(\pi_1,\pi_j)$, indicating folding of the space and thus finishing the proof.
\end{proof}
Proposition~\ref{prop:flatness} implies that the folding  occurs if $r_1$ (Eq.~(\ref{eq:cum_max})) decreases at least once along the path.
In the next step, we show that the folding measure is neither sub- nor super-additive, i.e., it neither holds that $\chi(\Gamma_1)\oplus\chi(\Gamma_2)\le \chi(\Gamma_1\oplus\Gamma_2)$ nor $\chi(\Gamma_1)\oplus\chi(\Gamma_2)\ge \chi(\Gamma_1\oplus\Gamma_2)$ for  connected paths $\Gamma_1, \Gamma_2$. Indeed, consider $\Gamma_1=\{\pi_1,\pi_2\}$ and $\Gamma_2=\{\pi_2,\pi_3, \pi_4\}$. A counter example for the sub-additivity is a path traversing the activation regions defined as
\begin{equation}
\label{counter_subadditivity}
    \pi_1=\begin{pmatrix}
    0\\0\\0
\end{pmatrix},
\pi_2=\begin{pmatrix}
    0\\0\\1
\end{pmatrix},
\pi_3=\begin{pmatrix}
    1\\1\\1
\end{pmatrix},
\pi_4=\begin{pmatrix}
    1\\0\\1
\end{pmatrix},
\end{equation}
where $\chi(\Gamma_1\oplus\Gamma_2)=\frac14$ and $\chi(\Gamma_1)+\chi(\Gamma_2)=0+\frac13=\frac13$, thus $\chi(\Gamma_1)+\chi(\Gamma_2)\ge\chi(\Gamma_1\oplus\Gamma_2)$ (for connected paths $\Gamma_1$ and $\Gamma_2$). To see that we can also construct  a counter example for super-additivity, consider paths as previously with the activation patterns defined as 
\begin{equation}
\label{eq:counter_superadditivity}
    \pi_1=\begin{pmatrix}
    0\\0\\0
\end{pmatrix},
\pi_2=\begin{pmatrix}
    1\\1\\1
\end{pmatrix},
\pi_3=\begin{pmatrix}
    0\\0\\1
\end{pmatrix},
\pi_4=\begin{pmatrix}
    1\\0\\0
\end{pmatrix},
\end{equation}
Then, $\chi(\Gamma_1\oplus\Gamma_2)=\frac47$ while 
$\chi(\Gamma_1)+\chi(\Gamma_2)=0+\frac12=\frac12$, thus $\chi(\Gamma_1)+\chi(\Gamma_2)\le\chi(\Gamma_1\oplus\Gamma_2)$.

While nor super- nor sub-additivity hold for every path $\Gamma$, in our experiments we have  only  observed  sub-additivity of the folding measure. The counterexample for super-additivity (Eq.~(\ref{eq:counter_superadditivity})), seems to be a rare occurrence in trained networks, though can be observed in specially constructed examples (see CantorNet by~\citet{lewandowski2024cantornet}).
The general lack of super- or sub-additivity, but empirical sub-additivity motivates us to introduce the  deviation from additivity for two paths $\Gamma_1$ and $\Gamma_2$ as $\Delta:\mathcal{H}^{n_1\cdot N}\times \mathcal{H}^{n_2\cdot N}\to[0,1]$, where 
\begin{equation}
\label{eq:interaction_coeff}
\mathcal{I}(\Gamma_1,\Gamma_2) := |\chi(\Gamma_1\oplus \Gamma_2) - \chi(\Gamma_1) - \chi(\Gamma_2)|.
\end{equation}

\subsection{On the Directnedness of Folding} 
\label{sec:directnedness}

So far, we have elaborated on the properties of the folding measure $\chi$ given by Eq.~(\ref{eq:measureFold}). We note that a path $\Gamma=\{\pi_1,\ldots,\pi_n\}$ along which the measure is computed is \emph{directed}, i.e.,  the measure $\chi$ computed along the reversed path, $-\Gamma:=\{\pi_n, \ldots, \pi_1\}$, may reach different folding values than $\chi(\Gamma)$, which we phrase as Remark~\ref{remark:asymmetricity}.

\begin{remark}[\textbf{Asymmetry}]
\label{remark:asymmetricity}
Consider $\Gamma=\{\pi_1,\pi_2,\pi_3\}$, where $\pi_1=(000), \pi_2=(111), \pi_3=(001)$, and its reverse $-\Gamma=\{\pi_3,\pi_2,\pi_1\}$. Then, $r_2(\Gamma)=r_2(-\Gamma)$ but $r_1(\Gamma) = 3$ and  $r_1(-\Gamma) = 2$, thus  $\chi(\Gamma)\neq\chi(-\Gamma).$
\end{remark}
It can be shown that, while folding is direction-sensitive, flatness is direction-invariant, expressed as Corollary~\ref{lemma:looped_path}.

\begin{corollary}[\textbf{Flatness Invariance}]
    \label{lemma:looped_path}
      $\chi(\Gamma)=0$ if and only if $\chi(-\Gamma)=0$ for  a path $\Gamma=\{\pi_1,\ldots, \pi_n\}$.
\end{corollary}
\begin{proof}
Observe that it is sufficient to prove Corollary~\ref{lemma:looped_path} only in way direction as we can re-index the path $\Gamma$ to obtain its reverse. 
We use  Proposition~\ref{prop:flatness}: if $\chi(\Gamma)=0$, then $d_H(\pi_1,\pi_i)$ is non-decreasing, what also implies that along the  reversed path $\Gamma$ the Hamming distance $d_H(\pi_n,\pi_i)$ is non-decreasing, indicating that $\chi(-\Gamma)=0$.
\end{proof}

\subsection{Global Space Folding Measure}
We now adapt a new, global measure of folding. 
Consider a classification problem with classes $C=\{1,\ldots,L\}$. Suppose that we have computed the folding measure for every pair of samples between classes $C_i$ and $C_j$ by pairwise computation and taking the median of non-zero values $\chi_+(C_i,C_j)$. We propose  the average of inter-class\footnote{
We used the Mann-Whitney test~\citep{whitney1947mannwhitneytest} to compare intra- and inter-class median folding values in networks with low generalization error. A statistically significant difference (per thresholds in~\citet{Cohen1992powerprimer}) showed that inter-class folding values are higher, suggesting that the network folds space within each digit class for more efficient representation, thus justifying their separate analysis.} folding values, i.e.,
\begin{equation}
    \label{eq:global_folding}
    \Phi_\mathcal{N}:=\frac 1{(L-1)L}\sum_{C_i\neq C_j}\chi_+(C_i,C_j) \in[0,1]
\end{equation}
as the global folding measure.  We posit that the global folding as characterized by Eq.~(\ref{eq:global_folding}) is the same if computed between every pair of samples regardless of their class assignment. Remark that, for a dataset with as little as  $10^4$ data points (e.g., MNIST test set), computing folding values pairwise would require $10^4!$  computations, which is computationally prohibitive even for small networks.  As it has been observed in ~\citep{lewandowski25spacefolds} that  the folding values in smaller networks (totaling 60 hidden neurons) and  in larger architectures (totaling 600 neurons) remain approximately the same for a fixed number of layers (if the networks are trained to a low generalization error), we posit that, for classification problems, the global folding   is a \emph{feature} of the neural architecture $\mathcal{N}$ for which it has been computed, i.e., 
\begin{equation}
\label{eq:sum_folding_and_reverse_folding}
\Phi_{\mathcal{N}}
\xrightarrow[]{\substack{\text{no. samples}\to\infty \\ \text{no. neurons}\to\infty}}
\text{const}(\mathcal{N}).
\end{equation}
The constant values of folding with increasing size of the network has yet another consequence. It means that, although there is an increasing number of linear regions as indicated by the works which provide bounds on this number, e.g.,~\citep{montufar2014number,raghu2017expressive,serra2018bounding,HaninR19reluhavefewactivation}, the networks fold the space in a very similar manner if the generalization error is low, which we exploit in Section~\ref{sec:folding_regularization}. 

\subsection{Beyond ReLU}
\label{sec:beyond_relu}
Thus far we have proven several  properties of the folding measure $\chi$ and provided additional interpretations. In this section, we interpret a walk through activation regions in ReLU-based MLP as a walk traversing distinct equivalence classes, and then show how this extends to \emph{any} activation function. This makes our study directly applicable to vast range of activation functions, such as Swish~\citep{ramachandran2018searching}, GELU~\citep{hendrycks2016gaussian} or SwiGLU~\citep{shazeer2020glu}. We start by defining the input equivalence relationship for ReLU neural networks.
\begin{definition}%[Adopted from \cite{shepeleva2020relu}]
\label{def:equivalence_classes}
 We define the equivalence relation between two inputs $\mathbf{x}_1, \mathbf{x}_2$ with respect to a neural network $\mathcal{N}$ as
$$
\mathbf{x}_1 \sim_\mathcal{N} \mathbf{x}_2
\Longleftrightarrow d_H\left(\pi(\mathbf{x}_1), \pi(\mathbf{x}_2)\right)=0
$$
\end{definition}
For  ReLU neural networks the equivalence class 
$
[\mathbf{x}_1]_\mathcal{N}:=\left\{\mathbf{z} \in \mathbb{R}^m \mid \mathbf{z} \sim_\mathcal{N} \mathbf{x}_1\right\}
$
corresponds to a linear region  which contains point $\mathbf{x}_1$.
We now show that the relation  in Def.~\ref{def:equivalence_classes} is that of equivalence.
Indeed, \textit{reflexivity} holds as  $\mathbf{x}\sim \mathbf{x}\Rightarrow \pi(\mathbf{x})=\pi(\mathbf{x})\Rightarrow d_H(\pi(\mathbf{x}),\pi(\mathbf{x}))=0$, and vice-versa, $d_H(\pi(\mathbf{z}),\pi(\mathbf{x}))=0$ holds for all $\mathbf{z}$ such that $\mathbf{z}\in[\mathbf{x}]_\mathcal{N}$, which also contains $\mathbf{x}$. \textit{Symmetry} is straightforward to check, and \textit{transitivity} holds as $\mathbf{x}\sim\mathbf{y}$ and $\mathbf{y}\sim\mathbf{z}$ implies that $d_H(\pi(\mathbf{x}),\pi(\mathbf{y}))=0$ and $d_H(\pi(\mathbf{y}),\pi(\mathbf{z}))=0$ thus also $d_H(\pi(\mathbf{x}),\pi(\mathbf{z}))=0$, and inversely, 0 Hamming distances between $\pi(\mathbf{x})$ and $\pi(\mathbf{y})$ as well as $\pi(\mathbf{y})$ and $\pi(\mathbf{z})$ imply that $\mathbf{z}\in[\mathbf{x}]_\mathcal{N}$.
In the following, we will extend the above results to a richer class of activation functions.  While it is possible,  we lose the geometrical interpretation of equivalence classes as  ``linear regions''. Henceforth for the computation of the folding measure $\Phi_\mathcal{N}$, we consider a walk through input equivalence classes, not linear regions, thus  extending the applicability of the space folding measure to much wider class of neural architectures. In order to obtain binary activation vectors, we threshold the values on intermediate layers (after applying the activation function) in a similar way as with the ReLU function, i.e., for a vector of activation values $\mathbf{a}\in\mathbb{R}^n$ we create an  \textit{activation pattern}  by only considering strictly positive vs. non-positive activation values, and denoting them with $1$ and $0$, respectively. For monotonic activation functions, we obtain a disjoint partition of the input space, thus the equivalence relationship as defined in Def.~\ref{def:equivalence_classes} holds.

\section{Sampling Strategy}
\label{sec:SamplingAlgorithm}

In Proposition~\ref{prop:traversing_same_region}, we have shown the importance of an appropriate sampling for numerically computing the folding measure efficiently (see also Fig.~\ref{fig:sampling_strategies}). In this section, we introduce a 1D sampling strategy in the Hamming activation space, presented in Algorithm~\ref{alg:sampling_strategy}. Our algorithm is based on the Hamming distance between activation patterns along the path. Our method is parameter-free, straightforward to implement and  intuitive to understand. It is based on the following intuition. Starting from \(\mathbf{x}_1\), we move incrementally towards \(\mathbf{x}_2\), generating the next point \(\mathbf{x}_\text{next}\). Initially, we take a step of length \(\Delta_{\text{init}}\), compute the activation pattern under the network \(\mathcal{N}\), and measure the Hamming distance \(d_H(\pi_1, \pi_\text{next})\). If \(d_H(\pi_1, \pi_\text{next}) = 0\), we proceed without storing \(\pi_\text{next}\); if \(d_H(\pi_1, \pi_\text{next}) = 1\), we store \(\pi_\text{next}\). Lastly, if \(d_H(\pi_1, \pi_\text{next}) > 1\), we iteratively reduce \(\Delta\) until either \(d_H(\pi_1, \pi_\text{next}) = 1\) or \(\Delta\) reaches \(\Delta_\text{min}\). In the latter case, we accept \(\pi_\text{next}\) and continue moving toward \(\mathbf{x}_2\).  We note that adjacent activation regions may have the Hamming distance exceeding 1 -- the algorithm will work unaffected, but this issue highlights the importance of the choice of the minimal step size $\Delta_\text{min}$. % -- we do not halve the step rate until we find the activation pattern of the Hamming distance equal 1. If we do, we traverse multiple times the same activation region (Fig.~\ref{fig:sampling_strategies}).

\begin{figure}[htb]
    \centering
    \includegraphics[width=1\linewidth]{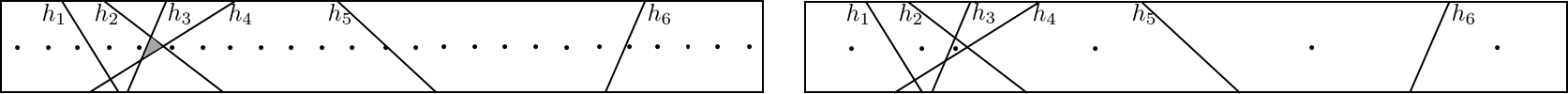}
    \caption{2D slice of the ReLU tessellation defined by hyperplanes $h_1,\ldots,h_6$ highlights the need for optimal sampling.  \textbf{Left:} Equally spaced points may revisit regions and miss small ones (gray). \textbf{Right:}  The optimized path visits each region exactly once.}
    \label{fig:sampling_strategies}
\end{figure}

\begin{algorithm}
    \caption{Sampling Strategy}
    \label{alg:sampling_strategy}
\begin{algorithmic}[1]
\STATE \textbf{input} points  $\mathbf{x}_1,\mathbf{x}_2$;  network $\mathcal{N}$; initial step size $\Delta_\text{init}$; minimal step size $\Delta_\text{min}$ 
%\newline\hspace{\algorithmicindent}This line has no number

\STATE \textbf{output} 
 $\mathcal{P}$:   activation patterns  between $\mathbf{x}_1$ and $\mathbf{x}_2$

\STATE $\Delta t \gets \Delta_{\mathrm{init}}$   %\frac1n||\mathbf{x}_1-\mathbf{x}_2||_2
\STATE $\pi_{\text{prev}} \gets \mathrm{GetActivationPattern}(\mathbf{x}_1)$
\STATE $\mathcal{P} \gets \{\, \pi_{\text{prev}} \}$

\WHILE{  $t < 1$}
  \STATE $t_{\mathrm{next}} \gets \min\,(t + \Delta t,\, 1)$ \;
  \STATE $x_{\mathrm{next}} \gets \mathbf{x}_1 + t_{\mathrm{next}} \,(\mathbf{x}_2 - \mathbf{x}_1)$ \;
  \STATE $\pi_{\mathrm{next}} \gets \mathrm{GetActivationPattern}(\mathbf{x}_{\mathrm{next}})$ \;
\IF{$d_H(\pi_{\text{prev}},\pi_{\text{next}})=1$}
\STATE $t,\pi_{\mathrm{next}}, \mathcal{P} \gets\mathrm{UpdateParams}(t, \pi_{\mathrm{prev}}, \mathcal{P})$
% \STATE $t \gets t_{\mathrm{next}}$ \;
% \STATE $p_{\text{prev}} \gets p_{\mathrm{next}}$ \;
% \STATE $\mathcal{P} \leftarrow \mathcal{P} \cup \{\, \pi_{\mathrm{next}} \}$ \;
\ELSIF{$d_H(\pi_{\text{prev}},\pi_{\text{next}})>1$}
\IF{$\Delta t \le \Delta_{\min}$} 
\STATE $t,\pi_{\mathrm{next}}, \mathcal{P} \gets\mathrm{UpdateParams}(t, \pi_{\mathrm{prev}}, \mathcal{P})$
% \STATE $t \gets t_{\mathrm{next}}$ \;
% \STATE  $\pi_{\text{prev}} \gets \pi_{\mathrm{next}}$ \;
% \STATE $\mathcal{P} \gets \mathcal{P} \cup \{\, \pi_{\mathrm{next}} \}$ \;
\ELSE
\STATE $\Delta t \leftarrow \Delta t / 2$ \;
\ENDIF
\ELSE
\STATE $t\gets t_\text{next}$
\ENDIF
\ENDWHILE
\end{algorithmic}
\end{algorithm}

\paragraph{Complexity Analysis.}

Let $M$ be the total number of steps actually taken (including refined steps), $O(\texttt{C})$ be the cost of running the network in the inference mode. Hence, the total computational cost is $O(M\cdot\texttt{C})$. In the worst case, if many boundaries are crossed in very small intervals, the step size  keeps halving, leading to a potentially large $M$.  Each halving leads to a geometric progression, resulting in  $\log \left(\Delta_{\text {init }} / \Delta_{\min }\right)$ refinement steps in some regions. In a typical scenario, $M$ might be on the order of a few hundred. In a pathological scenario, it can grow larger but is still upper-bounded by repeated halving.

\section{Folding as a Regularization Strategy}
\label{sec:folding_regularization}
Thus far, \citet{lewandowski25spacefolds} observed that increased space folding values correlate positively with the generalization capabilities of ReLU-based MLP, motivating  its use  as a regularization during the training process.   One approach is to modify the loss function as follows:
$$
Loss\gets Loss + \lambda \frac1{(\Phi_\mathcal{N}+1)^{2}}.
$$
This formulation takes advantage of the fact that  $\Phi_\mathcal{N}\in[0,1]$: During the early stages of the training when $\Phi_\mathcal{N}$ is low, the regularization effect is strong; as the network learns and $\Phi_\mathcal{N}$ increases, the influence diminishes. To further encourage early folding, we  use $(\Phi_\mathcal{N}+1)^2$. 
To  incorporate  $\Phi_\mathcal{N}$ into a gradient-based learning algorithm, we replace the non-differentiable maximum function in $r_1$ with a smooth approximation using the log-sum-exp function. For a temperature parameter $\beta>0$, define
$$
\widetilde{r_1}(\Gamma)=\frac{1}{\beta} \log \left(\sum_{i=1}^n \exp \left(\beta d_H\left(\pi_1, \pi_i\right)\right)\right)
$$
As $\beta$ increases, $\widetilde{r_1}(\Gamma)$ approaches the true maximum; for finite $\beta$, the function is smooth and differentiable. The resulting folding value  can be  incorporated into the loss function and optimized using backpropagation. The regularization procedure generalizes to any activation function through equivalence classes (not limited to ReLU-based MLP, see  Sec.~\ref{sec:beyond_relu}).

\section{Final Remarks}
\label{sec:discussion}

\paragraph{Future Work.}
% \label{sec:future_work}
We outline several  directions to pursue in relation to the folding measure: (\textit{i}) The proposed  regularization scheme in Section~\ref{sec:folding_regularization} remains untested -- in its current form, it requires multiple stops during training (every $n$ epochs), suggesting opportunities for further optimization; (\textit{ii}) While we have extended the space folding measure beyond the ReLU activation function, its behavior in other neural architectures necessitates further investigation; (\textit{iii}) Lastly, although we have defined the interaction effect (Eq.~(\ref{eq:interaction_coeff})), we have not used it in empirical evaluations. We intend to do so in context the of adversarial attacks.

\paragraph{Summary.}
Our study deepens the mathematical understanding of the space folding measure and lays the groundwork for further experimental work. We have extended the applicability of the space folding measure to any activation function,  highlighted key theoretical properties, and suggested its potential as a regularization technique.

\paragraph{Acknowledgments}
This research was carried out under the Austrian COMET program (project S3AI with FFG no. 872172), which is funded by the Austrian ministries BMK, BMDW, and the province of Upper Austria. 

\bibliography{references}
\bibliographystyle{iclr2025_conference}

\end{document}